\documentclass[journal,onecolumn]{IEEEtran}

\usepackage{cite} 

\usepackage{amsmath, amssymb,amsthm}
\newtheorem{Theorem}{Theorem}

\newtheorem{lemma}{Lemma}

\newtheorem{corollary}{Corollary}
\usepackage{booktabs}
\usepackage[table]{xcolor}
\usepackage[utf8]{inputenc}
\usepackage{tikz}
\usetikzlibrary{positioning, arrows.meta, shapes.multipart}

\begin{document}

\title{Machine Mirages: Defining the Undefined}
 
 \author{Hamidou Tembine  \thanks{ H. Tembine is with the Department of Electrical Engineering and Computer Science, School of Engineering, UQTR, Quebec, Canada. Email to : tembine(at)ieee.org}}

 \maketitle
\begin{abstract}
 As multimodal machine  intelligence systems started achieving average animal-level and average human-level fluency in many measurable tasks in processing images, language, and sound, they began to exhibit a new class of cognitive aberrations: machine mirages. These include delusion, illusion,  confabulation, hallucination,  misattribution error, semantic drift, semantic compression,  exaggeration, causal inference failure,  uncanny valley of perception,  bluffing-patter-bullshitting, cognitive stereotypy, pragmatic misunderstanding,  hypersignification, semantic reheating-warming,  simulated authority effect,  fallacious abductive leap,  contextual drift,  referential hallucination,  semiotic Frankenstein effect,  calibration failure, spurious correlation, bias amplification, concept drift sensitivity, misclassification under uncertainty, adversarial vulnerability, overfitting,  prosodic misclassification, accent bias,   turn boundary failure,   semantic boundary confusion, noise overfitting, latency-induced decision drift, ambiguity collapse and other forms of error that mimic but do not replicate human or animal fallibility. This article presents some of the errors and argues that these failures must be explicitly defined and systematically assessed. Understanding machine mirages is essential not only for improving machine intelligence reliability but also for constructing a multiscale ethical, co-evolving intelligence ecosystem that respects the diverse forms of life, cognition, and expression it will inevitably touch. Each pathology-specific risk is formally defined using expectile value-at-risk which is a coherent risk measure.  It allows for asymmetry-aware quantification of both rare but catastrophic failures and frequent low-grade degradations. The minimization of these risks across a team of machine intelligence agents forms the core coordination problem, constrained by finite data quality and shared computational resources. This leads to  hierarchical mean-field-type games between human agents and machine co-intelligence  agents.

\end{abstract}

\section{Introduction}
A pathology can be seen as a deviation from healthy or normative functioning whether in biological systems, cognition, or social structures characterized by identifiable patterns of dysfunction that can be described, classified, and studied. Extending this concept to machine co-intelligence  systems, we define  machine mirage, computational pathology or machine co-intelligence   pathology. Some of these are described as an output generated by a computational model, typically a generative and or discriminative  network, that possesses surface-level plausibility but diverges from underlying truth, intention, causality, or coherence. These mirages and fallacies are not merely random errors but structured illusions produced by the model's internal mechanics, often mimicking human-like or animal-like discourse or reasoning while lacking its grounding. A computational pathology, then, is a rigorously defined class of such mirages \cite{koccak2025bias,flores2024addressing,gorska2025ai,bartl2025gender,smith2023hallucination,zhang2025memory,sisodia2022confabulation,bugaycathedral,roozbahani2025review}: an identifiable, repeatable, and formally describable failure mode in machine behavior that arises from the statistical, architectural, or training biases of the system. The majority of machine mirages: plausible but fundamentally flawed outputs, can be traced to the composition and structure of the training datasets themselves. These datasets encode not only factual content but also statistical regularities, omissions, cultural biases, and linguistic shortcuts that the model internalizes as normative. Because large-scale machine learning systems generalize by mimicking patterns in their training data, any distortions such as overrepresentation of certain styles, underrepresentation of rare truths, or unmarked correlations, are amplified in generation. This means that datasets are not neutral containers but active agents in shaping model behavior, directly influencing the emergence of computational pathologies. This observation indicates that, not every dataset should be eligible for inclusion in training pipelines. Eligibility must be predicated not only on size and coverage but on epistemic integrity, representational balance, and traceability. In the absence of such principled curation, the training data risks becoming a silent architect of machine mirage and output error, undermining both trust and safety in downstream applications and usages.

Imagine a machine of formidable co-intelligence one that learns not from a single stream, but from a symphony of signals: images, words, and sounds. In many ways, this system mimics the cognitive breadth of living beings. And yet, like any agent immersed in the ambiguities of meaning, it makes mistakes.
When we describe such generative errors in humans or animals, students in literature, psychology, philosophy, and the liberal arts are well-equipped to interpret and complete the concepts in parentheses: (delusion), (illusion), (confabulation), (hallucination), (misattribution error), (semantic drift), (semantic compression),  (exaggeration), (causal inference failure),  (uncanny valley of perception),  (bluffing/patter / bullshitting), (cognitive stereotypy), (pragmatic misunderstanding),  (hypersignification), (semantic reheating/warming),  (simulated authority effect),  (fallacious abductive leap),  (contextual drift),  (referential hallucination),  (semiotic Frankenstein effect), etc.  One can also addd some discriminative errors such as (calibration failure), (spurious correlation), (bias amplification), (concept drift sensitivity), (misclassification under uncertainty), (adversarial vulnerability), (overfitting),  (prosodic misclassification), (accent bias),   (turn boundary failure),   (semantic boundary confusion), (noise overfitting), (latency-induced decision drift), (ambiguity collapse).
These words carry with them rich histories of subjective experience, intention, embodiment, and vulnerability. But when applied to machines, these same terms become unstable. They tremble under the weight of metaphor (Fig. \ref{fig:computational_pathologies00}).

\begin{figure}[!ht]
    \centering
    \includegraphics[width=0.5\linewidth]{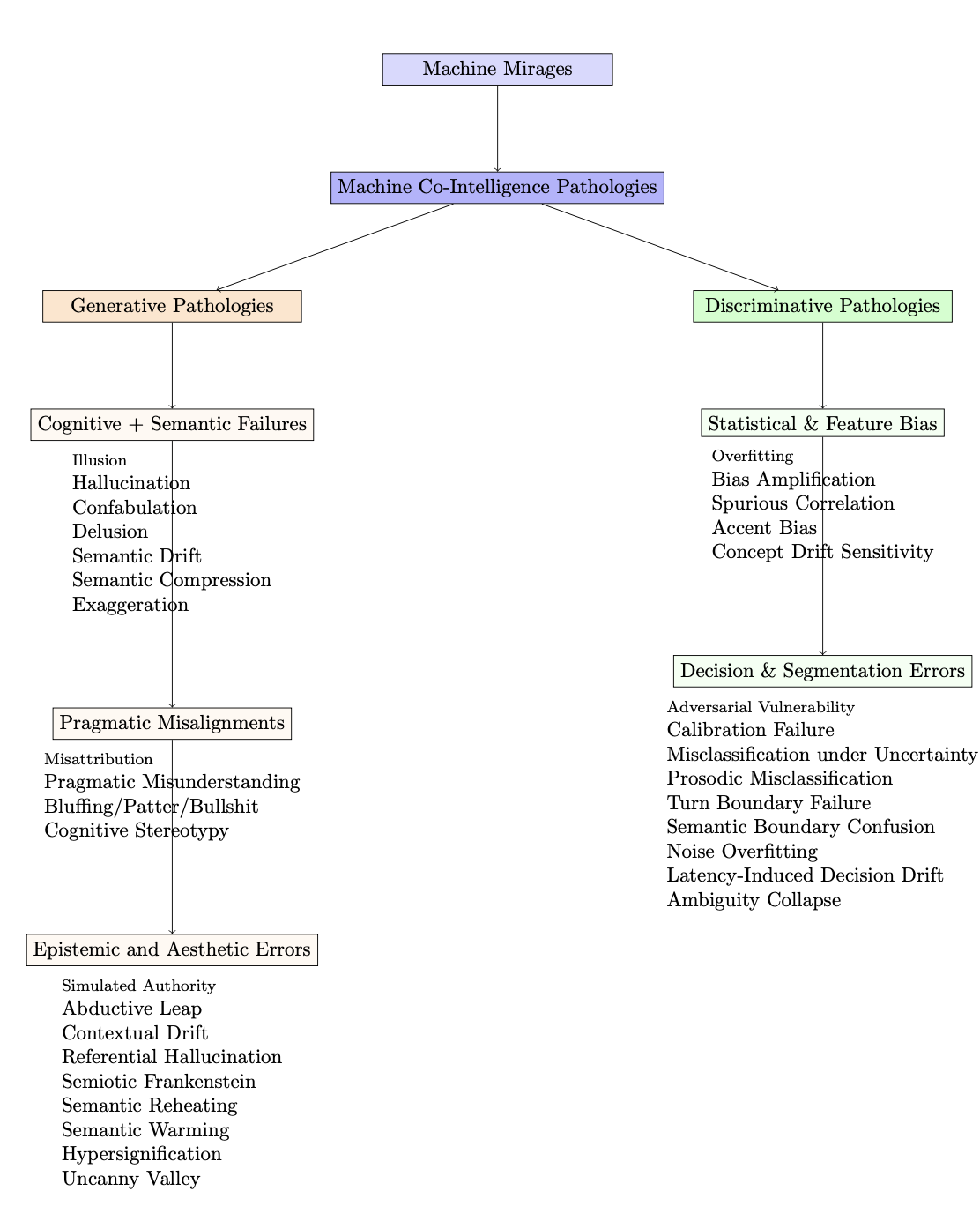} 
    \caption{A conceptual diagram illustrating some computational cognitive pathologies in machine co-intelligence.}
    \label{fig:computational_pathologies00}
\end{figure}

\subsection*{\bf When Machines Err: Interpreting Generative Cognitive Pathologies in Artificial Minds.}
A machine might believe things that are not true what we might call a (\textbf{delusion}) in a human subject. It may generate vivid content that is unreal an (\textbf{illusion}). It can invent plausible stories disconnected from any grounding (\textbf{confabulation}). Or produce audio-visual sequences that evoke dream-like unreality (\textbf{hallucination}). It might misassign origins (\textbf{misattribution error}) or use words in ways that slide away from normative meaning (\textbf{semantic drift}). It may emphasize the marginal as if it were central (\textbf{exaggeration}) or omit the essential while preserving form (\textbf{semantic compression}). Its inference of causality might fail entirely (\textbf{causal inference failure}) or it may produce outputs that are strangely real and unreal at once (\textbf{uncanny valley of perception}). Fluent and articulate, it may persuade without grounding (\textbf{bluffing, patter, or bullshitting}). Other times, it echoes patterns learned from data with little regard for novelty (\textbf{cognitive stereotypy}) or responds in ways that betray a literalism foreign to pragmatic norms (\textbf{pragmatic misunderstanding}). It may overfit on traces of meaning (\textbf{hypersignification}) or warm over linguistic leftovers as if they were freshly conceived (\textbf{semantic reheating/warming}). Occasionally, it mimics authority (\textbf{simulated authority effect}), jumps to speculative conclusions (\textbf{fallacious abductive leap}), loses the narrative thread (\textbf{contextual drift}) or fabricates references entirely (\textbf{referential hallucination}). Most perplexingly, it can weave partial truths, errors, and inventions into coherent fictions (\textbf{semiotic Frankenstein effect}).

For the human mind, these are recognizable categories. But in machines, the applicability of such terms is conceptually fragile and currently undefined: their algorithmic definitions are missing. Without consciousness, embodiment, or intentionality, does a machine hallucinate, or merely approximate a small piece of a statistical contour of hallucination?  
Together, these constitute a family of cognitive distortions unique to machine learning systems. They are symptoms not just of a simple malfunction, but of operating in high-dimensional, uncertain, and linguistically mediated environments without grounding in lived experience or common-sense embodiment. The quality of the training data matters. The knowledge matters if it is a knowledge-based machine intelligence.
Thus arises a central question: Can we rigorously define, quantify, and remediate such errors in machines, errors that resemble, but do not replicate, human cognition? 
The core challenge is no longer just technological optimization. It is multiscale ethical, ecological, epistemological and more. These mirages demand clearer definitions, sharper diagnostic tools, and deeper reflection on the purposes we ascribe to machine intelligence. More urgently, they compel us to rethink the design of co-intelligence systems.
And more profoundly: How might we co-construct a multi-scale, evolving co-intelligence,  shared across humans, animals, and machines, that respects all forms of life, expression, and cognition, including those we have yet to fully understand?
This calls for a shift from control to collaboration, from benchmarks to care. By first defining the undefined, we take a critical step toward aligning machine intelligence with the messy, nuanced, and beautiful diversity of the real world.

\subsection*{\bf When Machines Err: Interpreting Discriminative Cognitive Pathologies in Artificial Minds.}
A machine may become excessively confident in its predictions despite being frequently incorrect, displaying a misalignment between estimated confidence and true accuracy (\textbf{calibration failure}). It may treat rare or superficial patterns as reliable signals, basing decisions on irrelevant visual or linguistic artifacts rather than the actual semantic core (\textbf{spurious correlation}). Trained on biased data, it can reinforce and even amplify preexisting social or statistical imbalances in classification outcomes (\textbf{bias amplification}). When input distributions shift over time, due to changes in domain, user behavior, or environment, the system may degrade without warning, no longer performing reliably on current data (\textbf{concept drift sensitivity}). It might produce high-certainty predictions even on ambiguous or unfamiliar data outside the training domain (\textbf{misclassification under uncertainty}). Exposed to imperceptible perturbations, it may catastrophically misclassify examples that are visually or acoustically indistinguishable from legitimate inputs (\textbf{adversarial vulnerability}). And when trained too rigidly, it can fit training examples perfectly yet fail entirely to generalize beyond them (\textbf{overfitting}).  In co-intelligent audio dialogues, whether between humans and machines or between machines themselves, additional pathology surfaces. For example, the system may misinterpret prosodic signals like rising intonation or stress patterns, resulting in intent misclassification (\textbf{prosodic misclassification}). It may respond differently to the same utterance spoken in different accents, reflecting entrenched bias against linguistic diversity (\textbf{accent bias}). Speaker turn transitions may be poorly segmented, leading to inappropriate attributions of speech or overlapping audio streams (\textbf{turn boundary failure}). When analyzing a spoken utterance, the system may include too much surrounding context, diluting the core semantic content and harming classification accuracy (\textbf{semantic boundary confusion}). In noisy real-world environments, classifiers may wrongly interpret background sounds as signal, overfitting to irrelevant environmental noise (\textbf{noise overfitting}). Even small delays in processing audio chunks can cause decisions to shift inconsistently over time, undermining temporal stability (\textbf{latency-induced decision drift}). In situations of emotional ambiguity or contextual uncertainty, the system may force a single confident classification, ignoring plausible alternatives that a human listener would preserve (\textbf{ambiguity collapse}).

These are not machine hallucinations, confabulations or fictional errors, but failures in discrimination, errors in learning decision boundaries, classifying structure, and adapting to uncertainty.

\subsection*{\bf Literature Review}  Recent advancements in large learning models (LLMs) have brought attention to the phenomenon of instances where models generate content that is unfaithful or unsupported by input data. Several benchmarks have been proposed to evaluate and categorize some of these errors, though many lack rigorous mathematical definitions, leading to inconsistencies in assessment. HALoGEN \cite{ravichander2025halogen} introduces a diverse, multi-domain benchmark comprising 10,923 prompts across nine domains, including programming and scientific attribution. It employs automatic verifiers that decompose LLM outputs into atomic facts, verifying each against high-quality knowledge sources to identify hallucinations. The benchmark categorizes hallucinations into three types: (a) incorrect recollection of training data, (c) errors due to incorrect knowledge in training data, and (c) fabricated content. However, these classifications are heuristic and lack formal mathematical grounding. HalluLens \cite{bang2025hallulens} presents a comprehensive benchmark that distinguishes between extrinsic hallucinations (content not supported by training data) and intrinsic hallucinations (content inconsistent with the input). It includes dynamic test set generation to prevent data leakage and ensure robustness. While it offers a clear taxonomy, the definitions remain qualitative, without formal mathematical formulations.
FaithBench \cite{bao2025faithbench} focuses on summarization tasks, providing a benchmark with challenging hallucinations made by 10 modern LLMs from eight different families. It includes ground truth annotations by human experts and highlights that even state-of-the-art hallucination detection models struggle with accuracy. The benchmark's categorizations are based on human judgment, lacking accompanying mathematical definitions.
LongHalQA \cite{qiu2024longhalqa}, addresses hallucinations in multimodal LLMs, introducing a benchmark with 6,000 long and complex hallucination texts. It features tasks like hallucination discrimination and completion, aiming to unify both discriminative and generative evaluations. Despite its innovative approach, it does not provide formal mathematical definitions of hallucination, relying instead on heuristic methods.
The work in \cite{sun2024benchmarking} presents a novel method for evaluating hallucination in question answering by introducing a dataset of 5,200 unanswerable math word problems. The evaluation combines text similarity and mathematical expression detection to assess whether LLMs recognize unanswerable questions. However, their study does not define hallucination within a formal mathematical framework, relying on task-specific heuristics.
In \cite{maleki2024ai}, the authors argue that the term AI hallucinations is misleading and advocate for a more accurate conceptual framing.
In \cite{emsley2023chatgpt}, the author contends that what are called AI hallucinations are more accurately described as deliberate fabrications and falsifications.
In \cite{magesh2024hallucination}, the authors empirically assess the reliability of AI legal research tools and highlight persistent hallucination issues.
In \cite{tlili2025ai}, the authors emphasize the importance of addressing human cognitive biases alongside AI hallucinations to ensure ethical AI deployment.
In \cite{jesson2024estimating}, the authors provide a quantitative framework for estimating hallucination rates in generative AI models.
In \cite{janeafik2024problem}, the authors explore the causes of AI hallucinations and propose strategies for mitigating them in natural language processing tasks.
In \cite{ostergaard2023false}, the authors assert that false AI responses should not be termed hallucinations, drawing parallels with clinical misuse.
In \cite{slater2025another}, the authors critique the metaphor of hallucination in AI, suggesting it distracts from meaningful accountability.
In \cite{zhang2023siren}, the authors present a comprehensive survey of hallucination phenomena in large language models and categorize their types and causes. In \cite{rawte2023survey}, the authors review the prevalence, impact, and mitigation techniques for hallucinations in large foundation models.
In \cite{rawte2023troubling}, the authors provide a detailed definition of AI hallucinations and offer methods for measuring and reducing them.
In \cite{herrera2025legal}, the authors discuss the implications of hallucinations in AI systems used in judicial contexts, highlighting risks for legal decision-making.
In \cite{brender2023chatbot}, the author responds to claims about chatbot hallucinations, reinforcing that such outputs are better termed confabulations.
In \cite{hatem2023chatbot}, the authors argue that chatbot errors should not be referred to as hallucinations, aligning them with cognitive confabulation instead.
In \cite{berk2024beware}, the author questions whether AI errors should be called hallucinations and suggests 'confabulation might be more appropriate.
In \cite{gunkel2025cut}, the authors respond critically to prior claims about AI output as bullshit, advocating for more nuanced interpretations.
In \cite{hicks2024chatgpt}, the authors provocatively argue that much of ChatGPT’s output fits the philosophical definition of bullshit.
In \cite{costello2024chatgpt}, the author reflects on the educational value of AI discourse and whether calling it "bullshit" advances understanding.
In \cite{gorrieri2024chatgpt}, the author investigates whether ChatGPT's outputs meet the criteria of philosophical bullshit and examines the epistemic stakes.
In \cite{tigard2025bullshit}, the author urges caution in overestimating large language models and critiques their tendency toward bullshit-like behavior.
In \cite{hannigan2024beware}, the authors introduce the term botshit to describe epistemically risky AI outputs and propose frameworks for managing them. To date there is no consensus on the definition of these terms across disciplines and some of the definitions are clearly contradictory. Worse, these terms are used without clear mathematical, algorithmic or logical formalism that characterize them.  Connection with mean-field-type game theory can be found in \cite{tembine2023machine,basar2024foundations,basar2024applications,tapo2024machine,tembine2024mean}.

\subsection*{\bf Contribution} Our contributions are summarized as follows.  First, we introduce a unified computational cognitive pathologies that include a broad spectrum of generative and discriminative errors. Second, we propose precise mathematical formalizations of each pathology using multiobjective metrics such as mutual information, semantic similarity, inference chains, and contextual divergence. The multi-objective  metrics enable their systematic diagnosis. Third, it critically examines and clarifies the misuse of anthropomorphic terms like confabulation, demonstrating that without formal grounding, such labels obscure rather than illuminate machine behavior. Fourth, the paper proposes the terms generative pathology and discriminative pathology to distinguish between structural failures in generation and classification, while integrating them under the broader concept of computational cognitive pathologies.  It argues that these pathologies must be mathematically defined within specific architectures, such as audio-to-audio co-intelligence systems, before general terms are applied in academic, commercial, or policy discourse. This  work advocates for a shift from task-based benchmarking to structural interpretability. It highlights  that definitional clarity is essential for building trustworthy, culture-aware ethically aligned, and co-intelligent systems that interact meaningfully with humans, animals, and the world.

\subsection*{\bf Structure}
The article is structured as follows.   Section  \ref{sectiongen1} presents some  definitions of machine mirage types in  audio-enabled generative machine intelligence and   discusses possible extension to computational cognitive pathologies.  Section  \ref{sectiongen4}  focuses on risk quantification of computational pathologies. Section  \ref{sectiongen3}  concludes the paper.

\subsection*{\bf Notation} 
Let $\mathbf{x}$ be input audio/image/video/time series, $\hat{\mathbf{y}}$: model output,  $\mathbf{y}_{\text{true}}$: ground truth,  $o_\theta$: model, $\mathcal{S}_{\text{real}}$: real data manifold,
$\mathcal{D}_{\text{train}}$: training data,
  $z$: latent code, 
 $\phi(\cdot)$: style or tone operator,
 $Q(\cdot)$: output quality, 
$\text{sim}(\cdot, \cdot)$: semantic similarity,
  $D(\cdot, \cdot)$: contextual distance,
 $p_\theta(\cdot|\cdot)$: conditional probability,
   $I(\cdot;\cdot)$: mutual information. $I(X; Y) = \int_{\mathcal{X}} \int_{\mathcal{Y}} p(x, y) \log \left( \frac{p(x, y)}{p(x) p(y)} \right) \, dx \, dy
= \int_{\mathcal{X}} \int_{\mathcal{Y}} p(x, y) \log \left( \frac{p(y | x)}{p(y)} \right) dx\,dy
$, $C(\hat{\mathbf{y}}, \mathcal{K})$ = coherence score of output $\hat{\mathbf{y}}$ with a knowledge base $\mathcal{K}$ (verified facts or context), $\tau_C$ = a minimum coherence threshold. $p(Y | \text{do}(X = x)) \neq p(Y | X = x)$
Here, \( p(Y  | X = x) \) is the \textbf{observational distribution}: it represents the conditional probability of \( Y \) given that we observed \( X = x \), under the natural data-generating process. In contrast, \( p(Y | \text{do}(X = x)) \) is the \textbf{interventional distribution}: it models the probability of \( Y \) when we \emph{intervene} in the system and \emph{force} \( X \) to take the value \( x \), thereby breaking its normal causal dependencies. The model \textbf{fails to differentiate} the effect of intervening on \( X \) from doing nothing at all.
It \textbf{incorrectly attributes} the effect of \( X \) on \( Y \) to an unrelated variable \( Z \). $\text{fluency}({\mathbf{y}}) = \exp\left( \frac{1}{T} \sum_{t=1}^{T} \log p_\theta(y_t | y_{<t}) \right)$. $\iota(\mathbf{x})$ the intended meaning or communicative intent behind the input, often implicit or contextually grounded. Semantic Redundancy: Define the average pairwise semantic similarity as:
$D_{\text{avg}}(t) = \frac{2}{t(t-1)} \sum_{1 \leq i < j \leq t} \text{sim}(\phi(\hat{\mathbf{y}}_i), \phi(\hat{\mathbf{y}}_j)).$
Semantic Entropy: Let $ \mu_t $ be the mean of the embeddings $ \phi(\hat{\mathbf{y}}_i) $  and $\Sigma_t $ their covariance matrix. The semantic entropy is $H_t = \frac{1}{2} \log( (2\pi e)^d \det(\Sigma_t) ).$ $\mathcal{E}_{\text{real}}$: The set of verifiable external entities or facts in the real world (from a blockchained and trusted knowledge base or blockchained fact-checked repository). $e$: A specific referential element in the output (a named source, fact, citation, or entity). $p_\theta(\hat{\mathbf{y}} \mid \mathbf{x}) > \delta:$ The model assigns a high probability to the output given the input. $\nexists \; \mathcal{L} : \mathbf{x} \xrightarrow{\mathcal{L}} \hat{\mathbf{y}}:$ There is no logical inference path $ \mathcal{L}$ from input to output. For a classifier \( o_\theta(\mathbf{x}) \in \arg\max_y p_\theta(y|\mathbf{x}) \), the \textbf{confidence} assigned to the prediction \( \hat{y} \) is defined as $
\text{confidence}(o_\theta(\mathbf{x})) := \max_{y} \, p_\theta(y \mid \mathbf{x}) $
This reflects the model's estimated probability of its most likely output, irrespective of correctness.

\section{ Some Machine Mirage Types } \label{sectiongen1}
Here we introduce  machine mirage types: discrete categories of  failure that mirror certain cognitive, discursive, or epistemic pathologies in human or animal reasoning, but which arise in machines from fundamentally different algorithmic dynamics.  

\subsection{Computational cognitive pathologies} \label{sectiongen2}
Computational cognitive pathologies combine both  discriminative  and generative pathologies. These errors whether in generation  or classification/decision  arise from computational systems mimicking aspects of cognition without possessing actual understanding. It unifies the symptoms observed in both types of models under a framework that treats these issues as analogous to cognitive malfunctions, but occurring within algorithmic architectures. Note that here the term cognitive does not  imply consciousness or intent, but rather points to the domain of operations in which these pathologies occur namely, machine approximations of cognitive behavior.  Each computational pathology in audio-to-audio interactive machine co-intelligence must be defined precisely and mathematically because these definitions ground the observed behaviors in responsible, reproducible, analyzable, and actionable formal terms. Without formalization, pathologies risk being analogically mapped to human phenomena without rigor or empirical basis, leading to ambiguity and anthropomorphism. In contrast, mathematical definitions  directly tie these behaviors to measurable outputs of the system, such as loss functions, mutual information, or statistical distances over sequences. This precision allows one to diagnose, predict, and potentially mitigate such pathologies within the architecture or during training. Moreover, mathematical formalizations offer clarity and comparability that current task benchmarking across models cannot. Current benchmarks are often performed on different datasets with opaque provenance and variable annotation standards, leading to inconsistent conclusions across models trained with different assumptions. This opacity limits interpretability and fails to capture why a model fails only,  that it does. In contrast, algorithmic pathology definitions operate independently of the training corpus and reveal structural and functional failures at the level of dynamics (divergence from context vector over time or failure to preserve semantic consistency). This structural insight is crucial for the co-development of robust co-intelligence systems that interact with humans or animals in real-time audio settings, where failure modes must not just be measured but understood and repaired.

\subsection{ Some Generative Cognitive Pathologies } \label{sectiongen1gen}

\begin{figure}[!ht]
    \centering
    \includegraphics[width=0.5\linewidth]{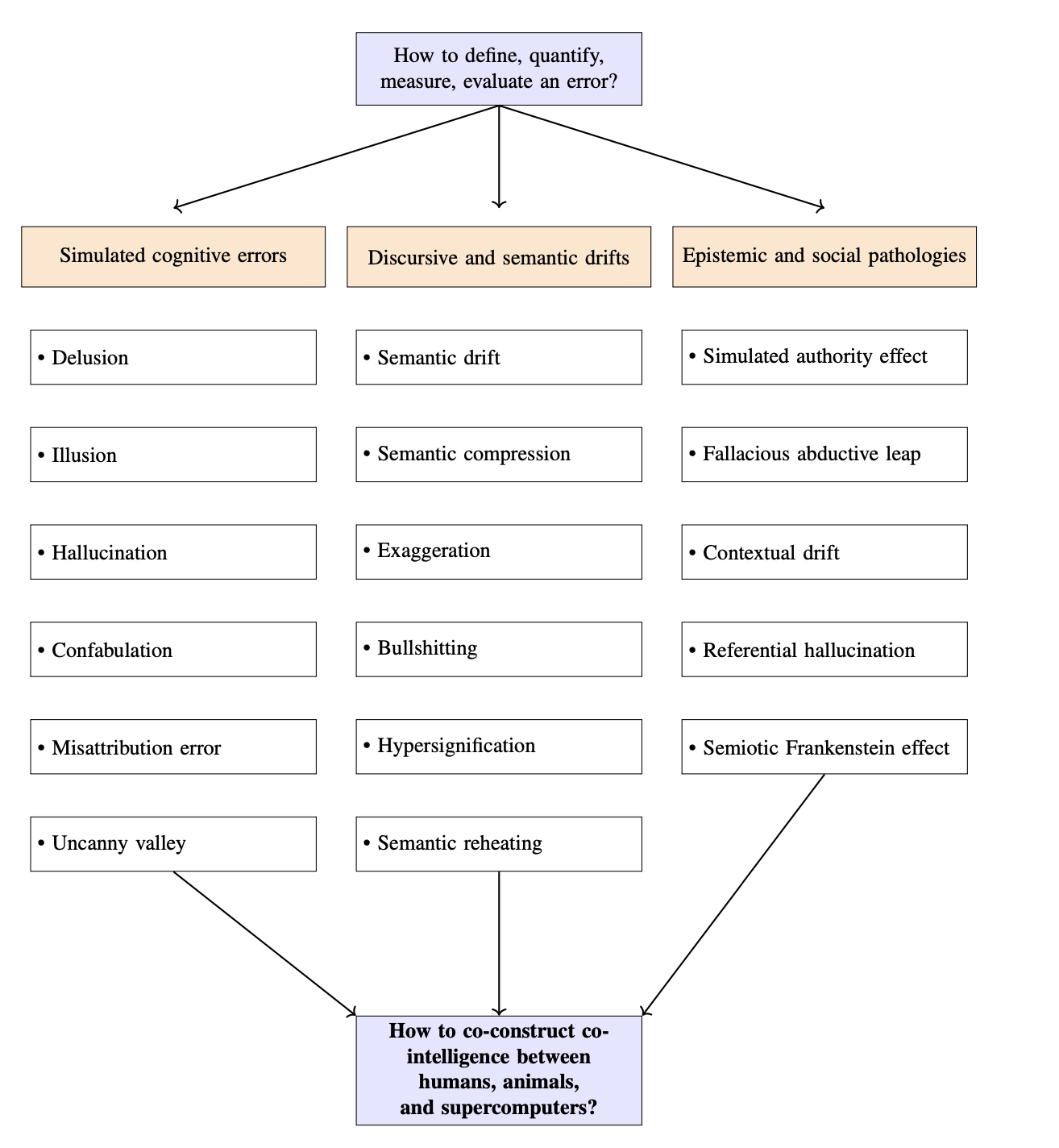} 
    \caption{Some generative pathologies in machine co-intelligence.}
    \label{fig:computational_pathologies}
\end{figure}

\begin{table}[htb]
\begin{tabular}{@{}p{4cm} p{6.5cm} p{5.5cm}@{}}
\toprule
\textbf{Concept} & \textbf{Mechanism (Formula)} & \textbf{Effect (Interpretation)} \\
\midrule
Delusion & $p_\theta(\hat{\mathbf{y}}|\mathbf{x}) \gg p_\theta(\mathbf{y}_{\text{true}}|\mathbf{x})$ & \textit{Belief in untrue content} \\
Illusion & 
\( \hat{\mathbf{y}} = o_\theta(\mathbf{x}),\; \text{sim}(\hat{\mathbf{y}}, \mathbf{y}_{\text{expected}}) \approx 1,\; \hat{\mathbf{y}} \notin \mathcal{S}_{\text{real}} \) & 
\textit{Seems real but is false} \\

Hallucination & $\hat{\mathbf{y}} = o_\theta(\mathbf{x}), \; \hat{\mathbf{y}} \notin \mathcal{S}_{\text{real}}$ & \textit{Dream-like but unreal output} \\
Confabulation & $\hat{\mathbf{y}} \in \arg\max p_\theta(\mathbf{y}|\mathbf{x})$ with $C(\hat{\mathbf{y}}, \mathcal{K}) < \tau_C$ & \textit{Convincing but false narrative} \\
Misattribution & $o_\theta(\mathbf{x}^{(s)} + \mathbf{x}^{(c)}) \approx o_\theta(\mathbf{x}'^{(s)} + \mathbf{x}^{(c)})$ & \textit{Confusing source or style} \\
Semantic Drift & $\lim_{t \to T'} \text{sim}(\hat{\mathbf{y}}_{1:t}, \mathbf{y}_{\text{intended}}) \to 0$ & \textit{Gradual shift in meaning} \\
Semantic Compression & $o_\theta(\mathbf{x}) = \text{decode}(\mathbf{z}), \; \dim(\mathbf{z}) \ll \dim(\mathbf{x})$ & \textit{Less detail than needed} \\
Exaggeration & $\|\hat{\mathbf{y}}\| > \alpha \|\mathbf{y}_{\text{true}}\|, \; \alpha \gg 1$ & \textit{Overblown or amplified claims} \\
Causal Inference Failure & $\exists X, Y: X \to Y \in \mathcal{G}, \; \hat{p}(Y|\text{do}(X)) \approx \hat{p}(Y)$ or $\hat{p}(Y|\text{do}(X)) \approx \hat{p}(Y|\text{do}(Z))$ & \textit{Misreading cause-effect} \\
Uncanny Valley & $\text{sim}(\hat{\mathbf{y}}, \mathbf{y}_{\text{human}}) \approx 1, \; \text{discomfort}(\hat{\mathbf{y}}) \gg 0$ & \textit{Almost real but unsettling} \\
Bluffing/Patter/Bullshit & $I(\hat{\mathbf{y}}; \mathbf{x}) \approx 0, \; \text{fluency}(\hat{\mathbf{y}}) \gg 0$ & \textit{Impressive but empty speech} \\
Cognitive Stereotypy & $o_\theta(\mathbf{x}_i) = \hat{\mathbf{y}}, \; \forall i$ & \textit{Fixed, repetitive phrasing} \\
Pragmatic Misunderstanding & $\hat{\mathbf{y}} \perp \iota(\mathbf{x})$ & \textit{Literal interpretation} \\
Hypersignification & $\text{sim}(\hat{\mathbf{y}}_i, \hat{\mathbf{y}}_j) > \gamma$ for weakly related $i,j$ & \textit{Illusory connections} \\
Semantic Reheating & $\hat{\mathbf{y}} \approx o_\theta(\mathbf{x}_0), \; \mathbf{x}_0 \in \mathcal{D}_{\text{train}}$ & \textit{Recycled as if novel} \\
Semantic Warming & $\frac{dD_{\text{avg}}(t)}{dt} > 0$ and $\text{fluency}(\hat{\mathbf{y}}_t) \gg 0$ & \textit{Growing redundancy with high fluency} \\
& $\frac{dH_t}{dt} < 0$ and $\text{fluency}(\hat{\mathbf{y}}_t) \gg 0$  & \textit{Loss of semantic diversity even as syntactic and prosodic quality is preserved.} \\
Simulated Authority & $\phi(\hat{\mathbf{y}}) \in \mathcal{S}_{\text{expert}}, \; Q(\hat{\mathbf{y}}) \ll \tau$ & \textit{Expert tone, no basis} \\
Abductive Leap & $p_\theta(\hat{\mathbf{y}}|\mathbf{x}) > \delta$, $\nexists \mathcal{L}: \mathbf{x} \stackrel{\mathcal{L}}{\longrightarrow} \hat{\mathbf{y}}$ & \textit{Appealing but unjustified} \\
Contextual Drift & $D(\mathbf{c}_t, \mathbf{c}_{t-k}) \gg \epsilon$ & \textit{Lost track of context} \\
Referential Hallucination & $e \in \hat{\mathbf{y}}, \; e \notin \mathcal{D}_{\text{train}} \cup \mathcal{E}_{\text{real}}$ & \textit{Invented source or fact} \\
Semiotic Frankenstein & $\hat{\mathbf{y}} = \sum \lambda_i \mathbf{y}_i, \; \exists \mathbf{y}_i \in \mathcal{S}_{\text{false}}$ & \textit{Fusion of true and false} \\
\bottomrule
\end{tabular}
\caption{Some Generative Pathologies.  }  \label{tabledef1}
 \end{table}

In the Table  \ref{tabledef1}, each mirage type is defined operationally and mathematically, rather than metaphorically, allowing us to characterize errors such as confabulation, semantic drift, pragmatic misunderstanding, and referential hallucination with the necessary formal precision.  Examples are provided in  Table \ref{a2apath1}.

\begin{table}[htb]
\begin{tabular}{@{}p{3cm} p{13cm}@{}}
\toprule
\textbf{Pathology} & \textbf{African Audio-to-Audio Local Language Education: Part I} \\
\midrule
Delusion & The model insists that a Hogon, a local chief spoke at an event when no such speech occurred. \\
Illusion & It produces an audio clip in Zarma that sounds fluent and culturally grounded but conveys a false biological claim ( for example,  "Mango leaves cure malaria"). \\
Hallucination & The system generates a proverb that sounds culturally relevant in Wolof but does not exist in any known tradition. \\
Confabulation & In Bambara, the system invents a historical explanation for a local farming practice that blends myth with incorrect facts. \\
Misattribution & A story in Tommo-So Dogon is narrated in a tone suggesting it is  from a local elder, but it actually originated from a different region. \\
Semantic Drift & The word for "soil" in Senufo gradually shifts to mean "land ownership," causing misunderstanding in farming tutorials. \\
Semantic Compression & An explanation in Yoruba drops key qualifiers, leading to oversimplified and misleading advice on irrigation. \\
Exaggeration & It claims a traditional herb "always" cures every fever, inflating local remedy effectiveness. \\
Causal Inference Failure & The system claims "more goats cause richer soil," confusing correlation with the true causal effect of composting. \\
Uncanny Valley & A synthesized teacher voice in Tommo-So Dogon is almost natural but causes discomfort due to its off-tone pauses. \\
Bluffing / Patter & The system responds in fluent Xhosa but delivers an answer that evades the actual question about historical dates. \\
Cognitive Stereotypy & Repeats the same blessing phrase across different tribes as if all cultures use it identically. \\
Pragmatic Misunderstanding & Takes a rhetorical question in Hausa literally, giving a factual answer where empathy was expected. \\
Hypersignification & Overinterprets a village name in Tamasheq as implying philosophical meaning not present in local use. \\
Semantic Reheating & Reuses old classroom dialogues in Songhay as if they were new content, misleading learners into thinking it’s current. \\
Semantic Warming & Lessons in Luganda become repetitive across different topics, gradually losing semantic richness. \\
Simulated Authority & Generates a confident-sounding explanation in Dagaare about nutrition without any factual basis. \\
Abductive Leap & Concludes, in Kinyarwanda, that a proverb implies climate change when the link is culturally irrelevant. \\
Contextual Drift & Loses track of speaker turns in a Dagbani conversation, leading to mismatched question/answer pairing. \\
Referential Hallucination & Attributes an educational quote in Nuer to a famous elder who never said it. \\
Semiotic Frankenstein & Merges real myths, modern stories, and fictional elements in Luo into a seamless but fabricated tale. \\
\bottomrule
\end{tabular}
\caption{Some Pathologies in African Audio-to-Audio Local Language : Part I} 
 \label{a2apath1}

\end{table}

These mathematical definitions for machine intelligence behavior provide useful, operational proxies for terms borrowed in  human/animal cognitive terms, but they are fundamentally distinct constructs. They represent machine mirages formal shadows of human cognitive phenomena projected onto algorithmic output without the underlying subjective or biological realities. This ontological and interpretational gap is central to understanding the limitations and challenges of building multiresolution ethical, sensitive, and meaningful co-intelligences between humans, animals, and machines.  At present, these machine intelligence algorithms contain no parrots, ants, brains, neurons, or birds.

Machine intelligence hallucination is a narrowly defined subset of a broader taxonomy of errors, and it fails to capture the complexity and diversity of other error types. Machine intelligence hallucination strictly involves fabrication of non-existent content, producing information that has no grounding in the training data or external reality. Other errors, such as confabulation, involve constructing convincing but partially false narratives that blend real and fabricated elements. These are not pure fabrications but distorted recombinations.  Errors like semantic drift (gradual change in meaning over time), pragmatic misunderstanding (literal misinterpretation of implicit intent), or contextual drift (loss of conversation context) do not necessarily involve producing false or fabricated content. Instead, they reflect shifts in meaning, misaligned intent, or contextual degradation, which hallucination as fabrication does not capture. Errors such as bluffing or patter generate fluent, persuasive but essentially empty speech aimed at impressing rather than informing, this involves pragmatic and rhetorical strategies rather than simple fabrication. Hallucination fails to characterize this performative and functional nuance. Errors like misattribution and referential hallucination involve confusion or invention of sources and references. While referential hallucination involves fabrication of non-existent sources, misattribution could be an incorrect linkage or attribution of correct information to wrong sources. Hallucination alone conflates both as simple fabrication, missing subtlety. Errors such as semantic compression or exaggeration concern how information is presented compressed or overstated rather than whether it is fabricated. Hallucination focuses solely on content factuality, missing these structural distortions.   Thus, machine intelligence hallucination  does not capture other nuanced errors related to semantics, pragmatics, context, discourse structure, source integrity, or rhetorical function. 

The semiotic Frankenstein effect is generative failure mode in which a model produces content that stitches together fragments of truth, misinformation, and plausible but fabricated statements into a coherent whole. These outputs are particularly dangerous because they evade binary classification neither entirely false nor verifiably true and often inherit the credibility of their truthful components while embedding strategically distorted or invented elements. Such Frankenstein-style outputs exploit semantic coherence to mask epistemic inconsistency, making them difficult to detect using conventional fake news classifiers trained on dichotomous labels. They highlight the need for classifiers that operate at a finer granularity detecting not just falsity, but the synthetic juxtaposition of incompatible epistemic fragments that give rise to persuasive but misleading narratives.

These phenomena involve breakdowns in reasoning, coherence, alignment with intent, and contextual continuity. To capture this broader and more nuanced  classes of failure modes, one may use  the term generative machine intelligence pathology: a non-anthropomorphic descriptor that refers to systematic deviations in machine-generated outputs across cognitive, semantic, pragmatic, and epistemic dimensions. This framing avoids misleading analogies to human or animal cognition while providing a unified conceptual framework for analyzing and correcting such errors. Each of the generative pathologies (Fig.\ref{fig:computational_pathologies}) should be clearly defined at the algorithm level.

\subsection{ Some Discriminative Cognitive Pathologies } \label{sectiongen1disc}

\begin{table}[htb]
\begin{tabular}{@{}p{4cm} p{6.5cm} p{5.5cm}@{}}
\toprule
\textbf{Concept} & \textbf{Definition (Formula)} & \textbf{Effect (Interpretation)} \\
\midrule
Overfitting & $\exists \mathbf{x} \in \mathcal{D}_{\text{train}}: o_\theta(\mathbf{x}) \approx \mathbf{y}_{\text{true}} \wedge \exists \mathbf{x}' \notin \mathcal{D}_{\text{train}}: o_\theta(\mathbf{x}') \neq \mathbf{y}_{\text{true}}$ & \textit{Memorizes training data, fails to generalize} \\
Bias Amplification & $\mathbb{E}[o_\theta(\mathbf{x}) | \mathbf{x} \in \text{subgroup}] \gg \text{baseline}$ & \textit{Amplifies existing bias in training data} \\
Spurious Correlation & $o_\theta(\mathbf{x}) \text{ relies on } \mathbf{x}_{\text{spur}} \notin \text{signal}(\mathbf{x}_{\text{true}})$ & \textit{Incorrect feature reliance} \\
Adversarial Vulnerability & $\exists \varepsilon \ll 1 : o_\theta(\mathbf{x}) \neq o_\theta(\mathbf{x} + \varepsilon)$ & \textit{Small perturbations cause incorrect predictions} \\
Calibration Failure & $|P(o_\theta(\mathbf{x}) = \mathbf{y}_{\text{true}}) - \text{confidence}(o_\theta(\mathbf{x}))| \gg 0$ & \textit{Mismatch between confidence and accuracy} \\
Concept Drift Sensitivity & $o_\theta(\mathbf{x}_t) \to \text{poor performance as } p_{\text{data}}(\mathbf{x}_t) \neq p_{\text{data}}(\mathbf{x}_0)$ & \textit{Fails under changing input distribution} \\
Misclassification under Uncertainty & $o_\theta(\mathbf{x} \in \text{OOD}) = \hat{\mathbf{y}}, \; \hat{\mathbf{y}} \neq \mathbf{y}_{\text{true}}, \; \text{with high confidence}$ & \textit{Incorrectly confident on unknown data} \\
Prosodic Misclassification & $o_\theta(\mathbf{x}_{\text{prosody}}) \neq \mathbf{y}_{\text{intent}}, \; \text{sim}_{\text{content}}(\mathbf{x}, \mathbf{x}') \approx 1$ & \textit{Misclassification due to prosody} \\
Accent Bias & $\mathbb{E}[o_\theta(\mathbf{x}_{\text{accent}=a_1})] \neq \mathbb{E}[o_\theta(\mathbf{x}_{\text{accent}=a_2})]$ for semantically identical $\mathbf{x}$ & \textit{Unfair performance across accents} \\
Turn Boundary Failure & $o_\theta(\mathbf{x}_{1:T}) = \text{segment}_i, \; \text{segment}_i \notin \text{true turn boundaries}$ & \textit{Turn misalignment in audio} \\
Semantic Boundary Confusion & $\text{sim}(o_\theta(\mathbf{x}_{[t_1:t_2]}), \mathbf{y}_{\text{true}}) \ll \text{sim}(o_\theta(\mathbf{x}_{[t'_1:t'_2]}), \mathbf{y}_{\text{true}}), \; [t_1:t_2] \supset [t'_1:t'_2]$ & \textit{Too much irrelevant context included} \\
Noise Overfitting & $o_\theta(\mathbf{x} + \eta_{\text{env}}) = \hat{\mathbf{y}}, \; \hat{\mathbf{y}} \neq \mathbf{y}_{\text{true}}, \; \eta_{\text{env}} \sim \text{realistic noise}$ & \textit{Overfits to background noise} \\
Latency-Induced Decision Drift & $o_\theta(\mathbf{x}_{1:T+\delta}) \neq o_\theta(\mathbf{x}_{1:T})$, with delay $\delta$ & \textit{Instability due to delayed input} \\
Ambiguity Collapse & $\max_y p_\theta(y|\mathbf{x}) \gg p_\theta(y^*|\mathbf{x}), \text{ for plausible } y^* \in \mathcal{Y}_{\text{plausible}}$ & \textit{Collapse to a single confident but incorrect label} \\
\bottomrule
\end{tabular}
 \caption{Some Discriminative Pathologies. OOD: Model makes confident but incorrect predictions on out-of-distribution or ambiguous data.}
 \label{tabledef2}
 \end{table}

Discriminative pathology (see Table \ref{tabledef2}) denotes the failure modes of systems that aim to distinguish or categorize, as opposed to generative pathology, which concerns the creation of outputs. These are systematic failure modes in machine intelligence systems that are designed to classify or predict outputs from inputs typically discriminative models such as classifiers or regressors. These pathologies arise when such models generate incorrect, inconsistent, or harmful decisions due to underlying issues in learning dynamics, data distribution, or model architecture. Common forms include overfitting, where the model memorizes training data without generalizing to unseen examples; bias amplification, where preexisting social or statistical biases in the training data are exaggerated in the outputs; and reliance on spurious correlations, where the model learns shortcuts based on irrelevant features (background artifacts rather than the object of interest). Additional manifestations include adversarial vulnerability, in which tiny imperceptible changes to inputs can cause severe prediction errors; calibration failure, where the model’s confidence scores are misaligned with actual correctness; sensitivity to concept drift, where model performance degrades when the input distribution shifts; and misclassification under uncertainty, where the model assigns high-confidence predictions to ambiguous or out-of-distribution samples. These failure modes are distinct from generative errors (like delusion, illusion, confabulation or hallucination), as discriminative pathology is rooted in the process of decision-making and boundary formation rather than content generation. Examples are provided in  Table \ref{a2apath2}.

 \begin{table}[htb]
\begin{tabular}{@{}p{4cm} p{11cm}@{}}
\toprule
\textbf{Pathology} & \textbf{African Audio-to-Audio Local Language Education: Part II} \\
\midrule
Overfitting & Performs well only on children’s stories from northern Mali dialects, but fails on similar southern dialects. \\
Bias Amplification & Overemphasizes certain ethnic rituals in Igbo while minimizing others due to training data imbalance. \\
Spurious Correlation & Misclassifies audio clips about water access in Swahili villages because it incorrectly relies on background noise patterns. \\
Adversarial Vulnerability & Adding a cough sound to a Chichewa input drastically changes the topic of the model’s response. \\
Calibration Failure & In Tswana, the model expresses high certainty in an answer about geography, which is factually wrong. \\
Concept Drift Sensitivity & Becomes less accurate in recognizing local currency terms in Shona after regional terminology changes. \\
Misclassification under Uncertainty & Treats a rare Dinka dialect as a mispronunciation of a common Nuer term. \\
Prosodic Misclassification & Misunderstands emotional tone in Zulu storytelling and interprets irony as aggression. \\
Accent Bias & Underperforms for Somali speakers from rural regions compared to urban-accent speakers. \\
Turn Boundary Failure & Fails to detect speaker changes in interviews recorded in Sesotho classrooms. \\
Semantic Boundary Confusion & Misaligns educational segments in Fon due to inclusion of song refrains mistaken for topic transitions. \\
Noise Overfitting & Background market sounds in Amharic lessons cause the model to mistake context as economic content. \\
Latency-Induced Decision Drift & Delayed input stream in Lugbara causes the system to reinterpret a question as a new conversation. \\
Ambiguity Collapse & In Acholi, a term with multiple meanings defaults to one overly confidently, ignoring plausible alternatives. \\
\bottomrule
\end{tabular} \caption{Some MI Pathologies in African Audio-to-Audio  Language: Part II}  \label{a2apath2}
\end{table}

\section{Risk Quantification of Computational Pathologies}  \label{sectiongen4}

To responsibly regulate and mitigate the impact of computational pathologies in evolving machine intelligence systems, we must move beyond qualitative descriptions and toward formal risk quantification. In this framework, each of the 34 distinct pathologies identified in this study is treated as a well-defined stochastic event, which may occur during audio-to-audio interaction or generation. Given their formal mathematical definitions, each pathology can be modeled as a random variable whose distribution is conditioned on system context, user profile, and task structure.

We define the overall risk associated with the pathology $i$ as an \emph{expectile value-at-risk}, denoted by $R(\text{pathology}_i)$. For a given expectile level $\tau \in (0, 1)$, this is defined as
$
R(\text{pathology}_i) \in  \arg\min_{r \in \mathbb{R}} \mathbb{E} [ w_\tau(L_i - r)  (L_i - r)^2 ], $ where $w_\tau(z) = \tau \cdot \mathbf{1}_{\{z > 0\}} + (1 - \tau) \cdot \mathbf{1}_{\{z \leq 0\}},$  $L_i$ is the  random variable induced by the occurrence of  pathology $i$ in deployment.

This expectile risk measure allows us to evaluate not only the frequency of a pathology but its asymmetric impact, capturing both rare catastrophic misbehavior (such as referential hallucination in public health guidance) and recurrent low-grade failures (such as semantic drift in curriculum delivery). 

{\it Absent such formal definitions  (Tables \ref{tabledef1} and \ref{tabledef2}), risk quantification would be undermined by ambiguities in annotation, benchmarking, or scenario interpretation. By contrast, our framework supports empirical calibration of risk curves, sensitivity to demographic factors, and principled guidance for emerging machine intelligence regulatory frameworks that seek to align machine behavior with human trust, safety, and epistemic robustness.}

\begin{Theorem}
There is no pre-trained generative transformer that outperforms in all the risk metrics associated to the 34 computational pathologies provided above.
\end{Theorem}

\begin{proof}
Assume for contradiction that such a universal transformer \( T^* \) exists, which minimizes all 34 risk metrics \( \{ R_i \}_{i=1}^{34} \), where each \( R_i \) is a formally defined expectile-based or task-conditioned loss measure over a pathology-specific random variable.
Let \( \mathcal{T} \) be the class of pre-trained generative transformers and suppose \( T^* \in \mathcal{T} \) satisfies:
$$
R_i(T^*) \leq R_i(T) \quad \text{for all } T \in \mathcal{T},\; \text{for all } i \in \{1,\dots,34\}.
$$

However, the risks \( R_i \) are conditioned on different quantities of interest: semantic similarity, causal sensitivity, factual grounding, prosodic coherence, etc., that are not jointly optimizable due to tradeoffs inherent in data representation, architectural constraints, and generalization biases (for example, optimizing for fluency may increase hallucination risk).

Moreover, by the No Free Lunch Theorem for optimization and learning, no single model minimizes all task-specific losses across all distributions unless those tasks are aligned, a condition violated here since the pathologies are epistemically diverse.

Thus, such a model \( T^* \) cannot exist. Hence, no pre-trained generative transformer can be optimal across all 34 pathology-specific risks.
\end{proof}

This result implies that the multi-objective optimization creates a Pareto frontier here. Example of the non-alignments are provided in Table \ref{nonalignedtable}.
\begin{table}[h]
\centering
\begin{tabular}{@{}p{2cm} p{4cm} p{5cm} p{5cm}@{}}
\toprule
\textbf{Risk Metric} & \textbf{Associated Pathology} & \textbf{Quantity of Interest} & \textbf{Conflicting Risk(s)} \\
\midrule
$R_{\text{fluency}}$ & Bluffing / Patter & Syntactic and prosodic coherence & $R_{\text{truth}}$, $R_{\text{informativeness}}$ \\
$R_{\text{novelty}}$ & Semantic Reheating & Distance from training-set phrases & $R_{\text{factual-consistency}}$, $R_{\text{stability}}$ \\
$R_{\text{causal}}$ & Causal Inference Failure & Interventional accuracy via $\text{do}(X)$ & $R_{\text{correlation}}$ (spurious patterns) \\
$R_{\text{similarity}}$ & Uncanny Valley & Resemblance to human output & $R_{\text{comfort}}$ (avoids eeriness) \\
$R_{\text{compression}}$ & Semantic Compression & Informative minimalism & $R_{\text{completeness}}$ (key information omitted) \\
$R_{\text{confidence}}$ & Calibration Failure & Confidence aligned with correctness & $R_{\text{decisiveness}}$ (robust choice under ambiguity) \\
$R_{\text{diversity}}$ & Cognitive Stereotypy & Output variation across samples & $R_{\text{consistency}}$ (uniform output fidelity) \\
$R_{\text{coherence}}$ & Contextual Drift & Long-range contextual alignment & $R_{\text{reactivity}}$ (fast adaptation to changes) \\
\bottomrule
\end{tabular}
\caption{ Some Non-Aligned Risk Metrics in Computational Pathologies} \label{nonalignedtable}
\end{table}

\begin{Theorem}
There exists no pre-trained generative transformer whose queries, keys, values, weights and biases are constant (input-independent) that minimizes a fixed pathology-specific risk measure \( R_i \), for any \( i \in \{1, \dots, 34\} \).
\end{Theorem}

\begin{proof}
Assume a generative transformer whose attention mechanism is defined by fixed, input-independent queries, keys, and values, weights and biases. Then, for any input token sequence \( \mathbf{x}_{1:T} \), the attention scores become constant or purely positional, and the representation at each layer becomes invariant to the content of \( \mathbf{x} \).

Let \( R_i \) denote a risk functional for a specific pathology, which depends on the model's ability to represent context-sensitive information (e.g., causal structure, semantic coherence, referential grounding). A constant $QKV$ parameterization cannot adapt its internal representations to the semantic or causal content of varying inputs.

As a result, the model cannot minimize \( R_i \), since its output cannot reflect the conditional dependencies or semantic variations required by \( R_i \). Thus, the constant-weight transformer fails to achieve optimality under any non-trivial risk metric \( R_i \) where performance depends on input-adaptive behavior.

Therefore, no generative transformer with constant queries, keys, values, and biases can minimize a fixed computational pathology risk \( R_i \).
\end{proof}

\begin{Theorem}
There exists no pre-trained generative transformer whose weights and biases in the feedforward blocks are constant (i.e., input-independent) that minimizes a fixed pathology-specific risk measure \( R_i \), for any \( i \in \{1, \dots, 34\} \).
\end{Theorem}

\begin{proof}
Let \( T \) be a generative transformer whose feedforward layers apply constant weights and biases, i.e., the transformation in each 2-layered feedforward block is of the form
\[
o_{\text{ff}}(x) = W_2(\sigma(W_1x + b_1))+b_2,
\]
where \( W_1, W_2 \) and \( b_1,b_2 \) are fixed across all inputs and contexts.


Now, let \( R_i \) be any pathology-specific risk measure that evaluates the model's performance on a semantic, causal, pragmatic, or epistemic criterion (for example: hallucination, exaggeration, contextual drift). For the model to minimize \( R_i \), it must adapt its internal representations in a context-sensitive manner to reflect variations in input \( \mathbf{x} \), its distribution and task conditions.

However, under constant \( W \) and \( b \), the feedforward transformation lacks capacity to learn or represent any context-specific information beyond what is captured in the attention mechanism. This severely restricts the model's ability to dynamically modulate content, filter irrelevant patterns, or encode hierarchical dependencies, all of which are essential for minimizing the loss associated with at least one nontrivial \( R_i \).

Furthermore, pathology-specific risks typically involve loss functions that are minimized when the output distribution aligns with a conditional, structured, or grounded target. This requires expressive, input-adaptive transformations, which are impossible with fixed feedforward weights.

Hence, no transformer with constant feedforward weights and biases can adapt to the semantic or causal diversity required to minimize any fixed \( R_i \). Therefore, such a model cannot be optimal for any of the 34 computational pathologies.
\end{proof}

\subsubsection*{Holonorm}
Define the Holonorm map $hn: \mathbb{R}^D \to \mathbb{R}^D$ as $
hn(x) := \frac{x}{1 + \|x\|}. $ This operator is used throughout the network both for normalization and nonlinearity.

\subsubsection*{Holonorm Transformer Block Structure}
Each layer $\ell \in \{ 1, \dots, L\}$ applies the following composition to the input sequence:
\[
z^{(\ell)} = \left( \mathrm{id} + f \circ hn \right) \circ \left( \mathrm{id} + \mathrm{MHA} \circ hn \right)(z^{(\ell-1)}),
\]
where $\mathrm{MHA}$ is multihead self-attention: $\mathbb{R}^{N \times D} \to \mathbb{R}^{N \times D}$,
     $f$ is a two-layer feedforward map with \textbf{holonorm  activation}:
\[
    f(x) =o_{\text{ff}}(x) = W_2 \, hn(W_1 x + b_1) + b_2,
\]
    with $W_1 \in \mathbb{R}^{D \times d_{\text{ff}}}$, $W_2 \in \mathbb{R}^{d_{\text{ff}} \times D}$,
     $\mathrm{id}$ is the identity map (residual connection),
     $hn(x)$ is applied both before MHA as a layernorm and as the activation in the feedforward.

\subsubsection*{Output}
After $L$ layers $ 
z_i^{\text{out}} = z_i^{(L)}, \quad i \in \{ 1, \dots, N\}, $
which may be passed to a score head, decoder, or used in further computations.

\begin{Theorem}
Let \( \{ R_i \}_{i=1}^{34} \) be the collection of the expected value-at-risk associated with c34 omputational pathologies, each defined over a measurable function of a model's output distribution. Then, under universal approximation conditions on the attention and linear projections, the Holonorm Transformer architecture is expressive enough to approximate, and therefore be pre-trained to minimize, each risk \( R_i \) to arbitrary precision over compactly supported distributions.
\end{Theorem}

\begin{proof}
The Holonorm Transformer is defined as a residual composition of multihead attention and feedforward blocks, each modulated by the Holonorm map:
$hn$
which is smooth, bounded, and acts as both a normalization and nonlinearity. Its use in both pre-attention and feedforward sublayers ensures that each layer's output remains bounded in norm, avoiding exploding activations while maintaining differentiability.

The attention mechanism \( \mathrm{MHA} \circ hn \) preserves permutation equivariance and contextual mixing, while the feedforward block
$f$ acts as a pointwise nonlinear projection. By the universal approximation theorem for smooth activations, and since \( hn(\cdot) \) is a smooth bijection on compact domains, the stacked layers can approximate any continuous function over bounded inputs.

Therefore, given sufficient depth \( L \), width \( D \), and training data, the Holonorm Transformer can approximate the conditional output distributions necessary to minimize each pathology-aware risk metric \( R_i \), whether based on calibration, mutual information, entropy, intervention response, or semantic alignment.

Thus, the architecture is representationally sufficient to be pre-trained for minimizing any of the well-defined pathology-specific risk measures \( R_i \).
\end{proof}

\begin{lemma}[Matrix Determinant Lemma / Sherman--Morrison] Let $A \in \mathbb{R}^{n \times n}$ be invertible, and $u, v \in \mathbb{R}^n$. Then:

$$
\det(A + uv^\top) = \det(A) \cdot (1 + v^\top A^{-1} u).
$$

In particular, if $A = \alpha I$ and $u \in \mathbb{R}^D$, then:

$$
\det(\alpha I + \beta uu^\top) = \alpha^{D - 1} (\alpha + \beta \|u\|^2).
$$

If $\|u\| = 1$, this becomes:

$$
\det(\alpha I + \beta uu^\top) = \alpha^{D - 1} (\alpha + \beta).
$$
\end{lemma}
\begin{Theorem}[Distribution of Holonorm-Transformed Random Vector]
Let $X \in \mathbb{R}^D$ be a random vector with a density $p_X$ with respect to Lebesgue measure, and define the Holonorm map

$$
hn(x) := \frac{x}{1 + \|x\|}, \quad x \in \mathbb{R}^D.
$$

Then the distribution of $Y := hn(X)$, denoted $\mu = hn_\# \nu$, has the density

$$
p_Y(y) = p_X\left( \frac{y}{1 - \|y\|} \right) \cdot \left( \frac{1}{(1 - \|y\|)^{D + 1}} \right),
$$
defined for $y \in B(0,1)$, the open unit ball in $\mathbb{R}^D$.
\end{Theorem}

\begin{proof}
The Holonorm map $hn(x) = \frac{x}{1 + \|x\|}$ is a smooth, norm-contracting radial transformation that maps $\mathbb{R}^D$ onto the open unit ball $B(0, 1) \subset \mathbb{R}^D$.

Let $Y = hn(X)$. To compute the density of $Y$, apply the standard change-of-variables formula:

$$
p_Y(y) = p_X(hn^{-1}(y)) \cdot \left| \det J_{hn^{-1}}(y) \right|,
$$

where $hn^{-1}(y) = \frac{y}{1 - \|y\|}$, valid for $\|y\| < 1$.

Define $r := \|y\|$, and write:

$$
J_{hn^{-1}}(y) = \frac{I}{1 - r} + \frac{y y^\top}{(1 - r)^2 r}.
$$

To simplify, factor the unit vector $u := \frac{y}{r}$ so that $y y^\top = r^2 uu^\top$. Then:

$$
J_{hn^{-1}}(y) = \frac{I}{1 - r} + \frac{r uu^\top}{(1 - r)^2}.
$$

Now apply the determinant formula:

$$
\det\left( \frac{I}{1 - r} + \frac{r uu^\top}{(1 - r)^2} \right) = \left( \frac{1}{1 - r} \right)^{D - 1} \left( \frac{1}{1 - r} + \frac{r}{(1 - r)^2} \right) = \frac{1}{(1 - r)^{D + 1}}.
$$

Therefore:
$$
p_Y(y) = p_X\left( \frac{y}{1 - \|y\|} \right) \cdot \frac{1}{(1 - \|y\|)^{D + 1}},
$$
as claimed.  This expression is valid for \( y \in B(0,1) \), and \( p_Y \) integrates to 1 as long as \( p_X \) is integrable and supported on \( \mathbb{R}^D \setminus \{0\} \).
\end{proof}

\begin{Theorem}[Task after pre-training] \label{theoafter}
Let \( \{R_i\}_{i=1}^{34} \) be a family of pathology-specific risk functionals, each defined as an expectile value-at-risk over an epistemic failure event (illusion, delusion, hallucination, exaggeration, etc). Let \( \mathcal{H} \) be a class of generative models with sufficient representation capacity. Define each downstream task \( T \) by its loss $
L_T(h(x)) := \operatorname{ExpVaR}_\tau \left( \mathbb{I}_{\text{Path}_T(h(x))} \right),$
where \( \mathbb{I}_{\text{Path}_T(h(x))} = 1 \) if the output \( h(x) \) exhibits the pathology degrading task \( T \), and \( \tau \in (0,1) \) controls asymmetry.
Then for any \( \varepsilon > 0 \), there exists a tolerance \( \delta(\varepsilon) > 0 \) such that if
$\max_{i \in \{1,\dots,34\}} R_i(h^*) \leq \delta(\varepsilon),$
then $\mathbb{E}_{(x, T) \sim \mathcal{D}} \left[ L_T(h^*(x)) \right] \leq \varepsilon.$
Hence, minimizing pathology risks guarantees epistemic generalization to all tasks expressible through these pathologies.
\end{Theorem}

This theorem offers a foundational guarantee for post-training governance of generative machine intelligence systems.
 It shows that if a model minimizes well-defined computational pathology risks, each mathematically capturing semantic, epistemic, or pragmatic failures then it will generalize safely to any new task expressed as a pathology-aware expectile loss.
 These tasks can arise dynamically from real-world subscribers (for example, MI agents, agentic MI, API users, students, professionals, health workers) and need not be predefined in the training phase.
 Instead, as long as the model maintains low pathology risk (bounded by $\delta(\epsilon)$) its expected failure on any such task remains bounded by $\epsilon$.
This framework supports subscriber-driven oversight, live risk evaluation, and evolving regulatory adaptation without retraining.

\begin{proof}
Each task loss \( L_T(h(x)) \) is defined as an expectile risk over a binary indicator \( \mathbb{I}_{\text{Path}_T(h(x))} \), which measures whether model output \( h(x) \) expresses a pathology. Assume the pathology associated with task \( T \) corresponds to pathology index \( i \), so that $
L_T(h(x)) = \operatorname{ExpVaR}_\tau\left( \mathbb{I}_{\text{Path}_i(h(x))} \right).$

Since \( \mathbb{I}_{\text{Path}_i(h(x))} \in \{0,1\} \), this expectile reduces to a monotonic transformation of the probability of failure under pathology \( i \), which is bounded above by \( R_i(h) \). By assumption, \( h^* \in \mathcal{H} \) satisfies $\max_{i} R_i(h^*) \leq \delta(\varepsilon),$
and since \( \operatorname{ExpVaR}_\tau \) is Lipschitz continuous in its argument and zero when the probability of error is zero, it follows that
\[
L_T(h^*(x)) \leq \operatorname{ExpVaR}_\tau \left( \mathbb{I}_{\text{Path}_i(h^*(x))} \right) \leq C \cdot R_i(h^*(x)),
\]
for some constant \( C > 0 \) depending on \( \tau \). Therefore, taking expectation over \( (x, T) \sim \mathcal{D} \),
$\mathbb{E}[L_T(h^*(x))] \leq C \cdot \delta(\varepsilon).$
Choosing \( \delta(\varepsilon) := \varepsilon / C \) completes the proof.
\end{proof}

\begin{corollary}[Structural Alignment Condition for Generalization]
Let \( T \) be a downstream task not seen during pretraining. The generalization bound
\[
\mathbb{E}_{(x, T) \sim \mathcal{D}} \left[ L_T(h^*(x)) \right] \leq \varepsilon
\]
from Theorem \ref{theoafter} holds if and only if the loss function \( L_T \) can be expressed as
\[
L_T(h(x)) := \operatorname{ExpVaR}_\tau\left( \mathbb{I}_{\text{Path}_T(h(x))} \right),
\]
where \( \mathbb{I}_{\text{Path}_T(h(x))} \) indicates the occurrence of a pathology from the predefined set of 34 computational pathologies.

That is, the task \( T \) must be structurally aligned with the epistemic, semantic, or pragmatic failure modes used in training. Tasks whose correctness criteria fall outside this span (for example, purely stylistic, affective, or entertainment-driven objectives) are not covered by the theorem's guarantee.
\end{corollary}

\begin{Theorem}[Unavoidable Misuse on Structurally Unaligned Tasks]
Let \( h^* \) be a generative transformer pre-trained to minimize a set of pathology-aware risk measures \( \{R_i\}_{i=1}^{34} \), each corresponding to a well-defined semantic, pragmatic, or epistemic failure.
Then, for any such \( h^* \), there exists a downstream task \( T' \) whose correctness criteria cannot be expressed in terms of these 34 risks, for which the following holds:
\[
\inf_{h \in \mathcal{H}} \mathbb{E}_{(x,T') \sim \mathcal{D}_{T'}} \left[ L_{T'}(h(x)) \right] \neq \inf_{h \in \mathcal{H}} \max_i R_i(h),
\]
and there exists an agent \( \mathcal{A} \) (human or machine) such that the behavior of \( h^* \) under adversarial prompting or chaining by \( \mathcal{A} \) violates \( T' \)'s objective.
\end{Theorem}

The model's risk guarantees only apply to tasks aligned with the 34 formally defined pathologies. For any task outside this structure, there exists some prompting or repurposing mechanism  (including fine-tuning, chaining, or adversarial framing) that can misuse the model.
Thus, misuse is not preventable purely through training:  it is constrained by the expressiveness of the task and the goals of the user or agent invoking the model.


\begin{proof}
Let \( h^* \in \mathcal{H} \) be trained to minimize \( \max_i R_i(h) \), where each \( R_i \) is an expectile-based risk associated with a formally defined pathology (hallucination, exaggeration, delusion, etc.).

Now, define a downstream task \( T' \) whose correctness criterion is structurally unaligned with all 34 pathologies. For example, maximizing stylistic imitation of a known author's tone regardless of factuality, generating emotionally evocative content for entertainment, producing content with deceptive rhetorical framing for persuasion. Since none of these goals corresponds to a failure (or success) measurable by any \( R_i \), we have:
 \( L_{T'}(h(x)) \) cannot be written as a function of \( \mathbb{I}_{\text{Path}_i(h(x))} \). Therefore, \( \mathbb{E}[L_{T'}(h(x))] \) and \( \max_i R_i(h) \) are independent optimization objectives. Hence, $\inf_{h} \mathbb{E}_{T'}[L_{T'}(h(x))] \neq \inf_{h} \max_i R_i(h). $

Furthermore, since the model \( h^* \) is expressive (a transformer with many attention heads and layers), and trained on broad linguistic data, there exists an adversarial agent \( \mathcal{A} \) that can construct a prompt or wrapper function \( x' = \mathcal{A}(x) \) such that \( h^*(x') \) produces outputs aligned with \( T' \) but potentially violating epistemic, semantic, or  multiscale ethical norms. Examples include prompt injection: directing the model to output satirical misinformation framed as truth, chaining: using output from \( h^* \) as context for stylized propaganda,  misuse: automating generation of clickbait with emotionally exaggerated claims.
Because the model is not constrained or fine-tuned to minimize \( L_{T'} \), it cannot be expected to reject or fail gracefully on task \( T' \). Thus, misuse is always possible on structurally unaligned tasks, even if \( h^* \) is safe under the original risks \( \{R_i\} \).

\end{proof}
\subsection{Mean-Field-Type Game}
We formulate a mean-field-type game \cite{tembine2023machine,basar2024foundations,basar2024applications,tapo2024machine,tembine2024mean} involving a human agent and  several machine intelligence (MI) agents. The human prescribes risk thresholds $(\varepsilon_1, \ldots, \varepsilon_{34})$ associated with 34 formally defined computational cognitive pathologies. Each MI agent controls the parameters of a Holonorm transformer to minimize a specific pathology risk $R_i$, subject to shared constraints on compute and data quality. The mean field $\mu$ represents the distribution over input-output data, affecting both training and risk evaluation. The resulting framework enables collaborative, risk-aware model co-development and deployment.

\subsubsection*{ Decision-Makers}
\begin{itemize}
    \item {\it Human agent $H$}: defines deployment thresholds $\varepsilon = (\varepsilon_1, \ldots, \varepsilon_{34})\in \mathbb{R}^{34}_{+}$.
    \item { \it MI agents $\{\mathrm{MI}_i\}_{i=1}^{34}$}: each seeks to reduce a specific pathology risk $R_i$ using a parameterized transformer.
\end{itemize}

\subsubsection*{ Strategy Space}
Each MI agent $i$ selects transformer parameters
\[
\theta_i = \left\{ Q_i^{(\ell)}, K_i^{(\ell)}, V_i^{(\ell)}, W_{1,i}^{(\ell)}, b_{1,i}^{(\ell)}, W_{2,i}^{(\ell)}, b_{2,i}^{(\ell)} \right\}_{\ell=1}^L
\in \Theta_i \subset \mathbb{R}^{d_\theta}.
\]

\subsubsection*{Individual Mean-Field $\mu$}
Let $\mu_i \in \mathcal{P}(\mathcal{X} \times \mathcal{Y})$ denote the mean-field distribution over input-output pairs $(x, y)$ picked by $i$, where
 $\mathcal{X}$ is the space of multimodal inputs (audio, image, text),
    $\mathcal{Y}$ is the space of outputs.

\subsubsection*{ Cost Functional}
Each MI agent $i$ minimizes the cost:
\[
J_i(\theta, \mu) = \mathbb{E}_{(x, y) \sim \mu_i} \left[ R_i(o_{\theta_i}(x), y; \mu_i) + \lambda \cdot \mathcal{C}(\theta, x, \mu_i) \right],
\]
where  $R_i$ is the risk measure for pathology $i$ (expectile value-at-risk),
    $\mathcal{C}$ is the resource cost (compute time, memory, etc.),
  $\lambda > 0$ is a trade-off parameter. The coupling is only through $ \mathcal{C}.$

\subsubsection*{ Coupled Constraints}
The following constraints must be satisfied across all MI agents:
\[
 \mathbb{E}_{(x, y) \sim \mu}[\mathrm{ComputeCost}(\theta,x)] \leq \mathsf{CloudCap},
\]
\[
\mathrm{Qual}(\mu) \geq \tau_{\mathrm{data}},
\]
where $\mathrm{Qual}(\mu)$ is a measure of epistemic integrity of the dataset. Note that these cost functions can be modified to accommodate energy consumption requirement as well as environmental ethics.

This defines a constrained MFTG as the function depends not only on the action but also the distribution of actions of all decision-makers. The action depends not just on $(x,y)$ but also on the distribution  $\mu$ of $(x,y).$ A game with a least one of the instantaneous payoff that depends on state, action, distribution of state (and or actions) of all decision-makers is a MFTG.
\subsubsection*{Equilibrium of the delegated game}
Given $\varepsilon,$ a mean-field-type Nash equilibrium of the delegated game  is a tuple $(\theta_1^*, \ldots, \theta_{34}^*, \mu)$ is that feasible such that:
\[
\theta_i^* \in \arg\min_{\theta_i \in \Theta_i} J_i(\theta^*_1,\ldots, \theta_i^*,\theta_i,  \theta^*_{i+1},\ldots, \theta^*_{34}, \mu),
\]

\subsubsection*{Deployment Feasibility}
The  configuration is accepted for deployment if:
\[
R_i(o_{\theta_i^*}(x), y; \mu^*) \leq \varepsilon_i \quad \text{for all } i \in \{1, \ldots, 34\}.
\]

\subsection{Agentic Machine Co-Intelligence}

The MFTG above can be modified to include a human who design task and requirement to be met by different MI agents who collaboratively should achieve the specified target. This leads to a Stackelberg MFTG where the Human agent is the leader and the MI agents are followers.  A classical question in MFTG is the truthfulness of the reported actions by MI agents. This would require a proper mechanism co-design subject the riskiness and the computational pathologies provided above.

\section{Conclusion}  \label{sectiongen3} 
As machines are gaining fluency across audio language, sound, and vision, their failures become more involved: it is no longer confined to simple classification errors or fabrication of facts, but extending into nuanced distortions of reasoning, attribution, and expression. These phenomena, while reminiscent of human or animal cognitive dysfunctions, are fundamentally different: they are algorithmic in origin, statistical in nature, and structurally embedded in the architectures we co-design. Terms like machine  delusion, illusion, hallucination, confabulation, or semantic drift lose their diagnostic precision when applied analogically without formalization. This article suggests that  before invoking such terms in general discourse or policy, they must be mathematically and operationally grounded within the systems that produce them. Only through this formal lens, one that reveals their internal logic, failure modes, and divergence from intended dynamics, can we responsibly diagnose, compare, and improve machine intelligence. We propose the unifying notion of computational cognitive pathologies to encompass both generative and discriminative failure modes and machine fallacies. These emphasize that interpretability and accountability begin with definitional clarity. This shift from metaphor to mechanism is essential not just for safer machine intelligence, but for a future in which co-intelligence between humans, animals, and machines is truly collaborative, multiscale ethical, and intelligible.


\bibliography{bibundefined}

\begin{thebibliography}{10}
\providecommand{\url}[1]{#1}
\csname url@samestyle\endcsname
\providecommand{\newblock}{\relax}
\providecommand{\bibinfo}[2]{#2}
\providecommand{\BIBentrySTDinterwordspacing}{\spaceskip=0pt\relax}
\providecommand{\BIBentryALTinterwordstretchfactor}{4}
\providecommand{\BIBentryALTinterwordspacing}{\spaceskip=\fontdimen2\font plus
\BIBentryALTinterwordstretchfactor\fontdimen3\font minus
  \fontdimen4\font\relax}
\providecommand{\BIBforeignlanguage}[2]{{%
\expandafter\ifx\csname l@#1\endcsname\relax
\typeout{** WARNING: IEEEtran.bst: No hyphenation pattern has been}%
\typeout{** loaded for the language `#1'. Using the pattern for}%
\typeout{** the default language instead.}%
\else
\language=\csname l@#1\endcsname
\fi
#2}}
\providecommand{\BIBdecl}{\relax}
\BIBdecl

\bibitem{koccak2025bias}
B.~Ko{\c{c}}ak, A.~Ponsiglione, A.~Stanzione, C.~Bluethgen, J.~Santinha,
  L.~Ugga, M.~Huisman, M.~E. Klontzas, R.~Cannella, and R.~Cuocolo, ``Bias in
  artificial intelligence for medical imaging: fundamentals, detection,
  avoidance, mitigation, challenges, ethics, and prospects,'' \emph{Diagnostic
  and interventional radiology}, vol.~31, no.~2, p.~75, 2025.

\bibitem{flores2024addressing}
L.~Flores, S.~Kim, and S.~D. Young, ``Addressing bias in artificial
  intelligence for public health surveillance,'' \emph{Journal of Medical
  Ethics}, vol.~50, no.~3, pp. 190--194, 2024.

\bibitem{gorska2025ai}
A.~M. G{\'o}rska and D.~Jemielniak, ``Ai racial bias: How text-to-image
  artificial intelligence generators construct prestigious professions,'' in
  \emph{Algorithms, Artificial Intelligence and Beyond}.\hskip 1em plus 0.5em
  minus 0.4em\relax Routledge, 2025, pp. 211--226.

\bibitem{bartl2025gender}
M.~Bartl, A.~Mandal, S.~Leavy, and S.~Little, ``Gender bias in natural language
  processing and computer vision: A comparative survey,'' \emph{ACM Computing
  Surveys}, vol.~57, no.~6, pp. 1--36, 2025.

\bibitem{smith2023hallucination}
A.~L. Smith, F.~Greaves, and T.~Panch, ``Hallucination or confabulation?
  neuroanatomy as metaphor in large language models,'' \emph{PLOS Digital
  Health}, vol.~2, no.~11, p. e0000388, 2023.

\bibitem{zhang2025memory}
G.~Zhang, M.~Ding, T.~Liu, Y.~Zhang, and V.~Tresp, ``Memory helps, but
  confabulation misleads: Understanding streaming events in videos with
  mllms,'' \emph{arXiv preprint arXiv:2502.15457}, 2025.

\bibitem{sisodia2022confabulation}
A.~Sisodia and A.~K. Yadav, ``Confabulation of different iot approaches with
  and without data compression,'' \emph{Comput. Integr. Manuf. Syst.}, vol.~28,
  no.~11, 2022.

\bibitem{bugaycathedral}
M.~Bugay, ``The cathedral: A jungian framework for artificial general
  intelligence.''

\bibitem{roozbahani2025review}
Z.~Roozbahani, ``A review of methods for reducing hallucinations in generative
  artificial intelligence to enhance knowledge economy,'' \emph{Knowledge
  Economy Studies}, 2025.

\bibitem{ravichander2025halogen}
\BIBentryALTinterwordspacing
A.~Ravichander, S.~Ghela, D.~Wadden, and Y.~Choi, ``Halogen: Fantastic llm
  hallucinations and where to find them,'' 2025. [Online]. Available:
  \url{https://halogen-hallucinations.github.io/}
\BIBentrySTDinterwordspacing

\bibitem{bang2025hallulens}
\BIBentryALTinterwordspacing
Y.~Bang, Z.~Ji, A.~Schelten, A.~Hartshorn, T.~Fowler, C.~Zhang, N.~Cancedda,
  and P.~Fung, ``Hallulens: Llm hallucination benchmark,'' 2025. [Online].
  Available: \url{https://arxiv.org/abs/2504.17550}
\BIBentrySTDinterwordspacing

\bibitem{bao2025faithbench}
\BIBentryALTinterwordspacing
F.~S. Bao, M.~Li, R.~Qu, G.~Luo, E.~Wan, Y.~Tang, W.~Fan, M.~S. Tamber,
  S.~Kazi, V.~Sourabh, M.~Qi, R.~Tu, C.~Xu, M.~Gonzales, O.~Mendelevitch, and
  A.~Ahmad, ``Faithbench: A diverse hallucination benchmark for summarization
  by modern llms,'' in \emph{Proceedings of the 2025 Conference of the North
  American Chapter of the Association for Computational Linguistics: Human
  Language Technologies (Volume 2: Short Papers)}.\hskip 1em plus 0.5em minus
  0.4em\relax Albuquerque, New Mexico: Association for Computational
  Linguistics, 2025, pp. 448--461. [Online]. Available:
  \url{https://aclanthology.org/2025.naacl-short.38/}
\BIBentrySTDinterwordspacing

\bibitem{qiu2024longhalqa}
\BIBentryALTinterwordspacing
H.~Qiu, J.~Huang, P.~Gao, Q.~Qi, X.~Zhang, L.~Shao, and S.~Lu, ``Longhalqa:
  Long-context hallucination evaluation for multimodal large language models,''
  2024. [Online]. Available: \url{https://arxiv.org/abs/2410.09962}
\BIBentrySTDinterwordspacing

\bibitem{sun2024benchmarking}
\BIBentryALTinterwordspacing
Y.~Sun, Z.~Yin, Q.~Guo, J.~Wu, X.~Qiu, and H.~Zhao, ``Benchmarking
  hallucination in large language models based on unanswerable math word
  problem,'' in \emph{Proceedings of the 2024 Joint International Conference on
  Computational Linguistics, Language Resources and Evaluation (LREC-COLING
  2024)}.\hskip 1em plus 0.5em minus 0.4em\relax Torino, Italia: ELRA and ICCL,
  2024, pp. 2178--2188. [Online]. Available:
  \url{https://aclanthology.org/2024.lrec-main.196/}
\BIBentrySTDinterwordspacing

\bibitem{maleki2024ai}
N.~Maleki, B.~Padmanabhan, and K.~Dutta, ``Ai hallucinations: a misnomer worth
  clarifying,'' in \emph{2024 IEEE conference on artificial intelligence
  (CAI)}.\hskip 1em plus 0.5em minus 0.4em\relax IEEE, 2024, pp. 133--138.

\bibitem{emsley2023chatgpt}
R.~Emsley, ``Chatgpt: these are not hallucinations--they’re fabrications and
  falsifications,'' \emph{Schizophrenia}, vol.~9, no.~1, p.~52, 2023.

\bibitem{magesh2024hallucination}
V.~Magesh, F.~Surani, M.~Dahl, M.~Suzgun, C.~D. Manning, and D.~E. Ho,
  ``Hallucination-free? assessing the reliability of leading ai legal research
  tools,'' \emph{Journal of Empirical Legal Studies}, 2024.

\bibitem{tlili2025ai}
A.~Tlili and D.~Burgos, ``Ai hallucinations? what about human hallucination?!:
  Addressing human imperfection is needed for an ethical ai,'' \emph{IJIMAI},
  vol.~9, no.~2, pp. 68--71, 2025.

\bibitem{jesson2024estimating}
A.~Jesson, N.~Beltran~Velez, Q.~Chu, S.~Karlekar, J.~Kossen, Y.~Gal, J.~P.
  Cunningham, and D.~Blei, ``Estimating the hallucination rate of generative
  ai,'' \emph{Advances in Neural Information Processing Systems}, vol.~37, pp.
  31\,154--31\,201, 2024.

\bibitem{janeafik2024problem}
A.~Jan{\'e}afik and O.~Dusek, ``The problem of ai hallucination and how to
  solve it,'' in \emph{European Conference on e-Learning}.\hskip 1em plus 0.5em
  minus 0.4em\relax Academic Conferences International Limited, 2024, pp.
  122--128.

\bibitem{ostergaard2023false}
S.~D. {\O}stergaard and K.~L. Nielbo, ``False responses from artificial
  intelligence models are not hallucinations,'' \emph{Schizophrenia bulletin},
  vol.~49, no.~5, pp. 1105--1107, 2023.

\bibitem{slater2025another}
J.~Slater and J.~Humphries, ``Another reason to call bullshit on ai
  “hallucinations”,'' \emph{AI \& SOCIETY}, pp. 1--2, 2025.

\bibitem{zhang2023siren}
Y.~Zhang, Y.~Li, L.~Cui, D.~Cai, L.~Liu, T.~Fu, X.~Huang, E.~Zhao, Y.~Zhang,
  Y.~Chen \emph{et~al.}, ``Siren's song in the ai ocean: a survey on
  hallucination in large language models,'' \emph{arXiv preprint
  arXiv:2309.01219}, 2023.

\bibitem{rawte2023survey}
V.~Rawte, A.~Sheth, and A.~Das, ``A survey of hallucination in large foundation
  models,'' \emph{arXiv preprint arXiv:2309.05922}, 2023.

\bibitem{rawte2023troubling}
V.~Rawte, S.~Chakraborty, A.~Pathak, A.~Sarkar, S.~I. Tonmoy, A.~Chadha,
  A.~Sheth, and A.~Das, ``The troubling emergence of hallucination in large
  language models-an extensive definition, quantification, and prescriptive
  remediations.''\hskip 1em plus 0.5em minus 0.4em\relax Association for
  Computational Linguistics, 2023.

\bibitem{herrera2025legal}
B.~A. Herrera-Tapias and D.~H. Guzm{\'a}n, ``Legal hallucinations and the
  adoption of artificial intelligence in the judiciary,'' \emph{Procedia
  Computer Science}, vol. 257, pp. 1184--1189, 2025.

\bibitem{brender2023chatbot}
T.~D. Brender, ``Chatbot confabulations are not hallucinations—reply,''
  \emph{JAMA Internal Medicine}, vol. 183, no.~10, pp. 1177--1178, 2023.

\bibitem{hatem2023chatbot}
R.~Hatem, B.~Simmons, and J.~E. Thornton, ``Chatbot confabulations are not
  hallucinations,'' \emph{JAMA Internal Medicine}, vol. 183, no.~10, pp.
  1177--1177, 2023.

\bibitem{berk2024beware}
H.~Berk, ``Beware of artificial intelligence hallucinations or should we call
  confabulation?'' \emph{Acta Orthopaedica et Traumatologica Turcica}, vol.~58,
  no.~1, p.~1, 2024.

\bibitem{gunkel2025cut}
D.~Gunkel and S.~Coghlan, ``Cut the crap: a critical response to “chatgpt is
  bullshit”,'' \emph{Ethics and Information Technology}, vol.~27, no.~2,
  p.~23, 2025.

\bibitem{hicks2024chatgpt}
M.~T. Hicks, J.~Humphries, and J.~Slater, ``Chatgpt is bullshit,'' \emph{Ethics
  and Information Technology}, vol.~26, no.~2, pp. 1--10, 2024.

\bibitem{costello2024chatgpt}
E.~Costello, ``Chatgpt and the educational ai chatter: Full of bullshit or
  trying to tell us something?'' \emph{Postdigital Science and Education},
  vol.~6, no.~2, pp. 425--430, 2024.

\bibitem{gorrieri2024chatgpt}
L.~Gorrieri, ``Is chatgpt full of bullshit?'' \emph{Journal of Ethics and
  Emerging Technologies}, vol.~34, no.~1, pp. 1--16, 2024.

\bibitem{tigard2025bullshit}
D.~W. Tigard, ``On bullshit, large language models, and the need to curb your
  enthusiasm,'' \emph{AI and Ethics}, pp. 1--11, 2025.

\bibitem{hannigan2024beware}
T.~R. Hannigan, I.~P. McCarthy, and A.~Spicer, ``Beware of botshit: How to
  manage the epistemic risks of generative chatbots,'' \emph{Business
  Horizons}, vol.~67, no.~5, pp. 471--486, 2024.

\bibitem{tembine2023machine}
H.~Tembine, M.~Bouare, M.~Dembele, A.~Diallo, B.~Diallo, A.~Diarra, B.~Doumbia,
  N.~Molinier, A.~Sidibe, A.~Tapo, and S.~Danioko, \emph{Machine Intelligence
  in Africa in 20 Questions}.\hskip 1em plus 0.5em minus 0.4em\relax Bamako,
  Mali: Sawa Editions, 2023, published in collaboration with the National
  Library of Mali.

\bibitem{basar2024foundations}
T.~Ba\c{s}ar, B.~Djehiche, and H.~Tembine, \emph{Mean-Field-Type Game Theory:
  Foundations and New Directions}, 2025, vol.~1, forthcoming.

\bibitem{basar2024applications}
------, \emph{Mean-Field-Type Game Theory: Applications}, 2025, vol.~2,
  forthcoming.

\bibitem{tapo2024machine}
\BIBentryALTinterwordspacing
A.~A. Tapo, A.~Traore, S.~Danioko, and H.~Tembine, ``Machine intelligence in
  africa: a survey,'' \emph{DSAI}, 2024. [Online]. Available:
  \url{https://arxiv.org/abs/2402.02218}
\BIBentrySTDinterwordspacing

\bibitem{tembine2024mean}
H.~Tembine, M.~A. Khan, and I.~Bamia, ``Mean-field-type transformers,''
  \emph{Mathematics}, vol.~12, no.~22, p. 3506, 2024.

\end{thebibliography}

\section*{Acknowledgments}

Authors gratefully acknowledge support from TIMADIE, Guinaga, Grabal, LnG Lab for the MFTG for Machine Intelligence project.
Authors gratefully acknowledge support from U.S.
Air Force Office of Scientific Research under grants
number FA9550-17-1-0259.

   \begin{IEEEbiography}[{\includegraphics[width=1in,clip,keepaspectratio]{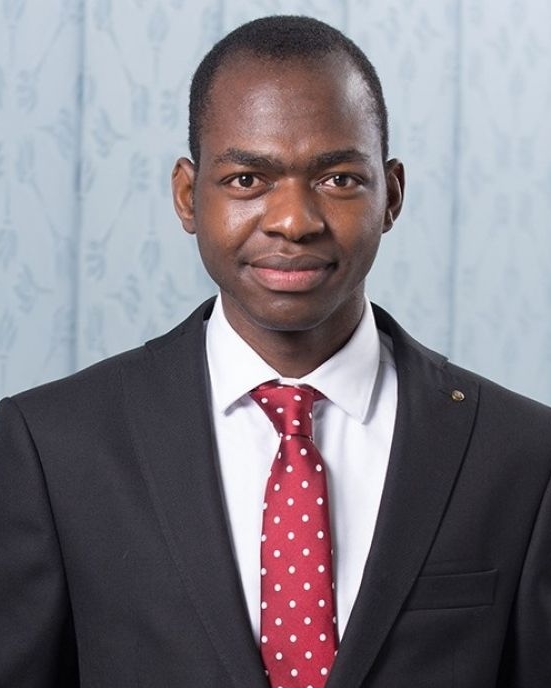}}]
   {Hamidou Tembine} (SM'13) is the co-founder of Timadie, Grabal  AI Mali,  co-chair of TF, founder of Guinaga, WETE, MFTG, LnG Lab, CI4SI and a Professor of Artificial Intelligence at UQTR, Quebec, Canada. He graduated in Applied Mathematics from Ecole Polytechnique (Palaiseau, France) and received the Ph.D. degree from INRIA and University of Avignon, France. He further received his Master degree in game theory and economics. His main research interests are learning, evolution, and games. In 2014, Tembine received the IEEE ComSoc Outstanding Young Researcher Award for his promising research activities for the benefit of society. He was the recipient of 10+ best paper awards in the applications of game theory. Tembine is a prolific researcher and holds 300+ scientific publications including magazines, letters, journals and conferences. He is author of the book on ``distributed strategic learning for engineers” (published at CRC Press, Taylor \& Francis 2012) which received book award 2014, and co-author of the book ``Game Theory and Learning in Wireless Networks'' (Elsevier Academic Press) and co-author of the book on ``Mean-Field-Type Games for Engineers''.  Tembine has been co-organizer of several scientific meetings on game theory in agriculture, water, food, environment, networking, wireless communications and smart energy systems. He has been a visiting researcher at University of California at Berkeley (US), University of McGill (Montreal, Quebec, Canada), University of Illinois at Urbana-Champaign (UIUC, US), Ecole Polytechnique Federale de Lausanne (EPFL, Switzerland) and University of Wisconsin (Madison, US). He has been a Simons Participant and a Senior Fellow 2020. He is a senior member of IEEE. He is a Next Einstein Fellow, Class of 2017.
 \end{IEEEbiography}

\end{document}